\documentclass[11pt]{article}
\usepackage{hyperref}
\usepackage{url}

\usepackage{algorithm}
\usepackage{algorithmic}
\usepackage{amsmath}
\usepackage{amsthm}
\usepackage{color}
\usepackage{graphicx}
\usepackage{booktabs}
\usepackage{fullpage}

\usepackage{times}

\title{Communication Efficient Distributed Agnostic Boosting}

\author{
Shang-Tse Chen \\
Georgia Institute of Technology\\ \small{schen351@gatech.edu} \and
Maria-Florina Balcan  \\
Carnegie Mellon University \\
\small{ninamf@cs.cmu.edu} \and
Duen Horng Chau  \\ 
Georgia Institute of Technology\\
\small{polo@gatech.edu}
}

\newtheorem{theorem}{Theorem}
\newtheorem{lemma}{Lemma}
\newtheorem{definition}{Definition}

\begin{document}

\maketitle

\begin{abstract}
We consider the problem of learning from distributed data in the agnostic setting, i.e., in the presence of arbitrary forms of noise.
Our main contribution is a general distributed boosting-based procedure for learning an arbitrary concept space, that is simultaneously noise tolerant, communication efficient, and computationally
efficient. This improves significantly over prior works that were either communication efficient only in noise-free scenarios or computationally prohibitive.
Empirical results on large synthetic and real-world datasets demonstrate the effectiveness and scalability of the proposed approach.
\end{abstract}

\section{Introduction}
Distributed machine learning has received an increasing amount of attention in this ``big data'' era~\cite{jordan2015machine}.
The most common use case of distributed learning is when the data cannot fit into a single machine, or when one wants to speed up the training process by utilizing parallel computation of multiple machines~\cite{alexps14,zdjw13,zdjw12}.
In these cases, one can usually freely distribute the data across entities, and an evenly distributed partition would be a natural choice.

In this paper, we consider a different setting where the data is inherently distributed across different locations or entities. Examples of this scenario include scientific data gathered by different teams, or customer information of a multinational corporation obtained in different countries. The goal is to design an efficient learning algorithm with a low generalization error over the union of the data.
Note that the distribution of the data from each source may be very different.
Therefore, to deal with the worst-case situation, we assume the data can be adversarially partitioned. This scenario has been studied for different tasks, such as supervised learning~\cite{BalcanBFM12,distfw15, Daume12}, unsupervised learning~\cite{disClustering13, balcan2014improved}, and optimization~\cite{Boyd2011, jaggi2014communication}.

Traditional machine learning algorithms often only care about sample complexity and computational complexity. However, since the bottleneck in the distributed setting is often the communication between machines~\cite{BalcanBFM12}, the theoretical analysis in this paper will focus on communication complexity.
A baseline approach in this setting would be to uniformly sample examples from each entity and perform centralized learning at the center.
By the standard VC-theory, a sampling set of size $O(\frac{d}{\epsilon^2}\log\frac{1}{\epsilon})$ is sufficient.
The communication complexity of this approach is thus $O(\frac{d}{\epsilon^2}\log\frac{1}{\epsilon})$ examples.

More advanced algorithms with better communication complexities have been proposed in recent works~\cite{BalcanBFM12,Daume12}. For example, \cite{BalcanBFM12} proposes a generic distributed boosting algorithm that achieves communication with only logarithmic dependence on $1/\epsilon$ for any concept class.
Unfortunately, their method only works in the standard realizable PAC-learning setting, where the data can be perfectly classified by a function in the hypothesis set and is noiseless.  This is because many boosting algorithms are vulnerable to noise~\cite{Dietterich00, long2008random}. The realizable case is often unrealistic in real-world problems.
Therefore, we consider the more general agnostic learning setting \cite{kearns94}, where there is no assumption on the target function. Since it is impossible to achieve an arbitrary error rate $\epsilon$, the goal in this setting is to  find a hypothesis with error rate close to $opt(H)$, the minimum error rate achievable within the hypothesis set $H$. The error bound is often in the form of $O(opt(H)) + \epsilon$.
Balcan et al. \cite{BalcanBFM12} propose an algorithm based on the robust generalized halving algorithm with communication complexity of $\tilde{O}(k\log({|H|})\log(1/\epsilon))$ examples. However, the algorithm works only for a finite hypothesis set $H$ and is computationally inefficient.

We propose a new distributed boosting algorithm that works in the agnostic learning setting. While our algorithm can handle this much more difficult and more realistic scenario, it enjoys the same communication complexity as in \cite{BalcanBFM12} that is logarithmic in  $1/\epsilon$ and exponentially better than the natural baselines. The algorithm is computationally efficient and works for any concept class with a finite VC-dimension.
The key insight, inspired by \cite{BalcanBFM12}, is that a constant (independent of $\epsilon$) number of examples suffice to learn a weak hypothesis, and thus if the boosting algorithm only needs $O(\log \frac{1}{\epsilon})$ iterations, we obtain the desired result.

A key challenge in this approach is that most agnostic boosting algorithms either have poor error bound guarantees or require too many iterations. The first agnostic boosting algorithm was proposed in \cite{ben2001agnostic}. Although the number of iterations is $O(\log\frac{1}{\epsilon})$ and is asymptotically optimal, their bound on the final error rate is much weaker: instead of $O(opt(H)) + \epsilon$, the bound is $O(opt(H)^{c(\beta)}) + \epsilon$, where $c(\beta)=2(1/2-\beta)^2/\ln(1/\beta -1)$. 
Some subsequent works \cite{Kanade09, Gavinsky03} significantly improve the bound on the error rate. However, their algorithms all require $O(1/\epsilon^2)$ iterations, which can in turn result in $O(1/\epsilon^2)$ communication in the distributed setting. Fortunately, we identify a very special boosting algorithm~\cite{Kale07} that runs in $O(\log \frac{1}{\epsilon})$ iterations.
This algorithm was analyzed in the realizable case in the original paper, but has later been noted to be able to work in the agnostic setting~\cite{Feldman10}\footnote{In the prior version of this paper which appeared in AISTATS 2016, we claimed that we were the first to show its guarantees in the agnostic setting. We thank the correction from the author of~\cite{Feldman10} that, although not explicitly proved or shown as a theorem, the feasibility of the algorithm in the agnostic setting has already been discussed in~\cite{Feldman10}.}.
We show how to adapt it to the distributed setting and obtain a communication efficient distributed learning algorithm with good agnostic learning error bound. 
Our main contributions are summarized as follows.

\begin{itemize}


\item We identify a centralized agnostic boosting algorithm and show that it can be elegantly adapted to the distributed setting. This results in the first algorithm that is both computationally efficient {\em and} communication efficient to learn a general concept class in the distributed agnostic learning setting.

\item Our proposed algorithm, which is a boosting-based approach, is flexible in that it can be used with various weak learners. Furthermore, the weak learner only needs to work in the traditional centralized setting rather than
in the more challenging distributed setting. This makes it much easier to design new algorithms for different concept classes in the distributed setting.

\item We confirm our theoretical results by empirically comparing our algorithm to the existing distributed boosting algorithm \cite{BalcanBFM12}. It does much better on the synthetic dataset and achieves promising results on real-world datasets as well.

\end{itemize}

\section{Problem Setup}
We first introduce agnostic learning as a special case of the general statistical learning problem. Then, we discuss the extension of the problem to the distributed setting, where the data is adversarially partitioned.  

\subsection{Statistical learning problem}
In statistical learning, we have access to a sampling oracle according to some probability distribution $D$ over $X\times \{-1,1\}$. The goal of a learning algorithm is to output a hypothesis $h$ with a low error rate with respect to $D$, defined as $err_D(h)=E_{(x,y)\sim D}(h(x)\neq y)$. Often, we compare the error rate to the minimum achievable value within a hypothesis set $H$, denoted by $err_D(H)=\inf_{h'\in H}err_D(h')$. More precisely, a common error bound is in the following form. 
\begin{equation}
\label{eq:agnostic}
err_D(h) \le c \cdot err_D(H) + \epsilon,
\end{equation}
for some constant $c\ge 1$ and an arbitrary error parameter $\epsilon>0$.

Many efficient learning algorithms have been proposed for the realizable case, where the target function is in $H$ and thus $err_D(H) = 0$.
In this paper, we consider the more general case where we do not have any assumption on the value of $err_D(H)$. This is often called the agnostic learning setting~\cite{kearns94} .
Ideally, we want $c$ in the bound to be as close to one as possible. However, for some hypothesis set $H$, achieving such a bound with $c=1$ is known to be NP-hard~\cite{Feldman09}.

\subsection{Extension to the distributed setting}
In this work, we consider the agnostic learning problem in the distributed learning framework proposed by \cite{BalcanBFM12}. In this framework, we have $k$ entities. Each entity $i\in[k]$ has access to a sampling oracle according to a distribution $D_i$ over $X\times \{-1,1\}$. There is also a \textit{center} which can communicate with the $k$ entities and acts as a coordinator.
The goal is to learn a good hypothesis with respect to the overall distribution $D=\frac{1}{k}\sum_{i=1}^k{D_i}$ without too much communication among entities.
It is convenient to calculate the communication by \textit{words}. For example, a $d$-dimensional vector counts as $O(d)$ words.\\

\noindent
\textbf{Main goal}. The problem we want to solve in this paper is to design an algorithm that achieves error bound~(\ref{eq:agnostic}) for a general concept class $H$. The communication complexity should depend only logarithmically on $1/\epsilon$.

\section{Distributed agnostic boosting}
In this work, we show a distributed boosting algorithm for any concept class with a finite VC-dimension $d$.
In the realizable PAC setting, the boosting algorithm is assumed to have access to a $\gamma$-weak learner that, under any distribution, finds a hypothesis with error rate at most $1/2 - \gamma$.
This assumption is unrealistic in the agnostic setting since even the best hypothesis in the hypothesis set can perform poorly.
Instead, following the setting of \cite{ben2001agnostic}, the boosting algorithm is assumed to have access to a $\beta$-weak agnostic learner defined as follows.
\begin{definition}
A $\beta$-weak agnostic learner, given any probability distribution $D$, will return a hypothesis $h$ with error rate
\[
err_{D}(h)\le err_{D}(H)+\beta.
\]
\end{definition}

Detailed discussion of the existence of such weak learners can be found in~\cite{ben2001agnostic}.
Since error of $1/2$ can be trivially achieved, in order for the weak learner to convey meaningful information, we assume $err_{D}(H) < 1/2 - \beta$.
Some prior works use different definitions. For example, \cite{kalai2008agnostic} uses the definition of $(\alpha, \gamma)$-weak learner.
That definition is stronger than ours, since an $(\alpha, \gamma)$-weak learner in that paper implies a $\beta$-weak learner in our paper with $\beta = \alpha - \gamma$. Therefore, our results still hold by using their definition. 
Below we show an efficient agnostic boosting algorithm in the centralized setting.

\subsection{Agnostic boosting: centralized version}
The main reason why many boosting algorithms (including AdaBoost \cite{FreundS97} and weight-based boosting~\cite{mehmohri,shai2}) fail in the agnostic setting is that they tend to update the example weights aggressively and may end up putting too much weight on noisy examples.

To overcome this, we consider a smoothed boosting algorithm~\cite{Kale07}, shown in Algorithm~\ref{alg:mwu}. This algorithm uses at most $O(\log 1/\epsilon)$ iterations and enjoys a nice ``smoothness'' property, which is shown to be helpful in the agnostic setting \cite{Gavinsky03}. The algorithm was originally analyzed in the realizable case 
but has later been noted to be able to work in the agnostic setting~\cite{Feldman10}.
Below, for completeness we show the analyses of the algorithm in both the realizable and agnostic settings. 


The boosting algorithm adjusts the example weights using the standard multiplicative weight update rule. The main difference is that it performs an additional Bregman projection step of the current example weight distribution into a convex set $\mathcal{P}$ after each boosting iteration. The Bregman projection is a general projection technique that finds a point in the feasible set with the smallest ``distance'' to the original point in terms of \textit{Bregman divergence}.
Here we use a particular Bregman divergence called \textit{relative entropy} $RE(p \parallel q)=\sum_i{p_i\ln(p_i/q_i)}$ for two distributions $p$ and $q$.
To ensure that the boosting algorithm always generates a ``smooth'' distribution, we set the feasible set $\mathcal{P}$ to be the set of all $\epsilon$-smooth distributions, which is defined as follows.

\begin{definition}
A distribution $D$ on $S$ is called $\epsilon$-smooth if $\max_i D(i)\le \frac{1}{\epsilon |S|}$,
\end{definition}

It is easy to verify that $\mathcal{P}$ is a convex set. The complete boosting algorithm is shown in Algorithm~\ref{alg:mwu} and the theoretical guarantee in Theorem~\ref{thm:smooth-realizable}. The proof, included in the appendix, is similar to the one in \cite{Kale07}, except that they use real-valued weak learners, whereas here we only consider binary hypotheses for simplicity.

\begin{algorithm}[t]
   \caption{Centralized Smooth Boosting algorithm \cite{Kale07}}
   \label{alg:mwu}
\begin{algorithmic}
   \STATE {\bf Initialization:} Fix a $\gamma\le\frac{1}{2}$. Let $D^{(1)}$ to be the uniform distribution over the dataset $S$.\\[0.5em]
   \FOR{$t=1,2,\dots,T$}
   \vskip2mm
   \STATE \begin{enumerate}
   \item Call the weak learner with distribution $D^{(t)}$ and obtain a hypothesis $h^{(t)}$\\[0.5em]
   \item Update the example weights
   \[
   \hat{D}^{(t+1)}(i) = D^{(t)}(i) \cdot (1-\gamma)^{\ell^{(t)}_i}/Z^{(t)}
   \]
   where $\ell^{(t)}_i = \mathbf{1}[h^{(t)} (x_i)=y_i]$ and $Z^{(t)}=\sum_i D^{(t)}(i) \cdot (1-\gamma)^{\ell^{(t)}_i}$ is the normalization factor.\\[0.5em]
   \item Project $\hat{D}^{(t+1)}$ into the feasible set $\mathcal{P}$ of $\epsilon$-smooth distributions
   \[
   D^{(t+1)} = \underset{D\in\mathcal{P}}{\arg\min}\; {RE(D\parallel \hat{D}^{(t+1)})}
   \]
   \end{enumerate}
   \ENDFOR
   \vskip0.5em
   \STATE {\bf Output:} The hypothesis $\mathrm{sign} \left(\frac{1}{T}\sum_{t=1}^T h^{(t)}\right)$
\end{algorithmic}
\end{algorithm}

\begin{theorem} \label{thm:smooth-realizable}
Given a sample $S$ and access to a $\gamma$-weak learner, Algorithm~\ref{alg:mwu} makes at most $T=O(\frac{\log(1/\epsilon)}{\gamma^2})$ calls to the weak learner with $\epsilon$-smooth distributions and achieves error rate $\epsilon$ on $S$.
\end{theorem}

Note that in Theorem~\ref{thm:smooth-realizable}, it is not explicitly assumed to be in the realizable case. In other words, If we have a $\gamma$-weak learner in the agnostic setting, we can achieve the same guarantee. However, in the agnostic setting, we only have access to a $\beta$-weak agnostic learner, which is a much weaker and more realistic assumption. The next theorem shows the error bound we get under this usual assumption in the agnostic setting.

\begin{theorem}\label{thm:smooth-agnostic}
Given a sample $S$ and access to a $\beta$-weak agnostic learner, Algorithm~\ref{alg:mwu} uses at most $O(\frac{\log(1/\epsilon)}{(1/2-\beta)^2})$ iterations and achieves an error rate $\frac{2err_{S}(H)}{1/2-\beta} + \epsilon$ on $S$, where $err_{S}(H)$ is the optimal error rate on $S$ achievable using the hypothesis class $H$.
\end{theorem}

\begin{proof}
The idea is to show that as long as the boosting algorithm always generates some $\epsilon'$-smooth distributions, the $\beta$-weak agnostic learner is actually a $\gamma$-weak learner for some $\gamma>0$, i.e., it achieves error rate $1/2-\gamma$ for any $\epsilon'$-smooth distributions.
In each iteration $t$, the $\beta$-weak agnostic learner, given $S$ with distribution $D^{(t)}$, returns a hypothesis $h^{(t)}$ such that

\begin{align*}
err_{D^{(t)}}(h^{(t)})& \le err_{D^{(t)}}(H) + \beta\\
                  & \le \frac{1}{\epsilon'} err_{S}(H) + \beta.
\end{align*}

The second inequality utilizes the $\epsilon'$-smoothness property of $D^{(t)}$. The reason is that if $h$ is the optimal hypothesis on $S$, we have
\[
err_{D^{(t)}}(H) \le err_{D^{(t)}}(h) \le \frac{\text{\#mistakes on }S}{\epsilon' |S|} = \frac{1}{\epsilon'} err_{S}(h) = \frac{1}{\epsilon'} err_{S}(H).
\]

Let $\frac{1}{\epsilon'} err_{S}(H) + \beta = \frac{1}{2}-\gamma$, or equivalently $\gamma = (\frac{1}{2} - \beta) - \frac{1}{\epsilon'} err_{S}(H)$. Then, if $\epsilon' \ge \frac{2err_{S}(H)}{1/2-\beta}$, we have $\gamma\ge\frac{1}{2}(1/2-\beta)>0$. Therefore, we can use Theorem~\ref{thm:smooth-realizable}, and achieves error rate $\epsilon'$ on $S$ by using $O(\frac{\log(1/\epsilon')}{(1/2-\beta)^2})$ iterations. Alternatively, it achieves error rate $\frac{2err_{S}(H)}{1/2-\beta}+\epsilon$ by using $O(\frac{\log(1/\epsilon)}{(1/2-\beta)^2})$ iterations.
\end{proof}

Next, we show how to adapt this algorithm to the distributed setting.

\subsection{Agnostic boosting: distributed version}
The technique of adapting a boosting algorithm to the distributed setting is inspired by \cite{BalcanBFM12}. They claim that any weight-based boosting algorithm can be turned into a distributed boosting algorithm with communication complexity that depends linearly on the number of iterations in the original boosting algorithm. However, their result is not directly applicable to our boosting algorithm due to the additional projection step.
We will describe our distributed boosting algorithm by showing how to simulate the three steps in each iteration of Algorithm~\ref{alg:mwu} in the distributed setting with $O(d)$ words of communication. Then, since there are at most $O(\log(1/\epsilon))$ iterations, the desired result follows.

In step 1, in order to obtain a $2\beta$-weak hypothesis (we use $2\beta$ instead of $\beta$ for convenience, which only affects the constant terms), the center calls the $\beta$-weak agnostic learner on a dataset sampled from $D^{(t)} = \frac{1}{k}\sum_{i=1}^k{D^{(t)}_i}$. The sampling procedure is as follows. 
Each entity first sends its sum of weights to the center. Then, the center samples $O(\frac{d}{\beta^2}\log\frac{1}{\beta})$ examples in total across the $k$ entities proportional to their sum of weights.
By the standard VC-theory, the error rate of any hypothesis on the the sample is within $\beta$ to the true error rate with respect to the underlying distribution, with high probability. It is thus sufficient to find a hypothesis with error within $\beta$ to the best hypothesis, which can be done thanks to the assumed $\beta$-weak learner. 

Step 2 is relatively straightforward. The center broadcasts $h^{(t)}$ and each entity updates its own internal weights independently. Each entity then sends the summation of internal weights to the center for the calculation of the normalization factor. The communication in this step is $O(kd)$ for sending $h^{(t)}$ and some numbers.
What is left is to show that the projection in step 3 can be done in a communication efficient way.
As shown in \cite{Herbster01}, the projection using relative entropy as the distance into $\mathcal{P}$, the set of all $\epsilon$-smooth distributions, can be done by the following simple algorithm. 

For a fixed index $m$, we first clip the largest $m$ coordinates of $p$ to $\frac{1}{\epsilon n}$, and then rescale the rest of the coordinates to sum up to $1-\frac{m}{\epsilon n}$. We find the least index $m$ such that the resulting distribution is in $\mathcal{P}$, i.e. all the coordinates are at most $\frac{1}{\epsilon n}$. A naive algorithm by first sorting the coordinates takes $O(n\log n )$ time, but it is communicationally inefficient.

Fortunately, \cite{Herbster01} also proposes a more advanced algorithm by recursively finding the median. The idea is to use the median as the threshold, which corresponds to a potential index $m$, i.e., $m$ is the number of coordinates larger than the median. We then use a binary search to find the least index $m$. The distributed version of the algorithm is shown in Algorithm~\ref{alg:project}.

\begin{algorithm}[!ht]
   \caption{Distributed Bregman projection algorithm}
   \label{alg:project}
\begin{algorithmic}
\vskip0.2em

   \STATE {\bf Input:}\\[1mm]
   \quad \textbf{Each entity $i$}:\\[1mm]
   \qquad a disjoint subset $\mathcal{W}_i$ of $\mathcal{W}=\{w_1,\dots,w_n\}$\\[1mm]
   \quad \textbf{Center}:\\[1mm]
   \qquad $n_0 = n$;\quad $C=0$;\quad $C^w=0$

   \vskip0.4em

   \WHILE{$n_0 \neq 0$}
   \vskip1mm
   \STATE distributedly find the median $\theta$ of $(\mathcal{W}_1,\dots,\mathcal{W}_k)$\\[2mm]

   \STATE  \textbf{Each entity $i$}: \\[0.1cm]
   \STATE \quad $\mathcal{L}_i=\{w: w < \theta, w\in\mathcal{W}_i\}$;\hskip0.9em  $\mathcal{M}_i=\{w: w = \theta, w\in\mathcal{W}_i\}$;\hskip0.9em $\mathcal{H}_i=\{w: w > \theta, w\in\mathcal{W}_i\}$\\[2mm]
   \STATE \quad $L^w_i = \sum_{w\in \mathcal{L}_i} w$;\hskip5.5em  $M^w_i= \sum_{w\in \mathcal{M}_i} w$;\hskip5.18em $H^w_i=\sum_{w\in \mathcal{H}_i} w$\\[2mm]
   \STATE  \textbf{Center}: \\[0.1cm]
   \quad $L = \sum_i |\mathcal{L}_i|$; \hskip2em $M = \sum_i |\mathcal{M}_i|$; \hskip2em $H = \sum_i |\mathcal{H}_i|$\\[2mm]
   \quad $L^w = \sum_i L^w_i$; \hskip1.57em $M^w = \sum_i M^w_i$; \hskip1.7em $H^w = \sum_i H^w_i$ \\[2mm]

   \STATE \quad $m_0=\frac{1-(C+H)\frac{1}{\epsilon n}}{1-(C^w+H^w)}$ and broadcasts it\\[2mm]
   \STATE \quad \textbf{if} $\theta m_0>\frac{1}{\epsilon n}$ \textbf{then}\\[1mm]
   \STATE \qquad $C = C + H + M$; \quad $C^w = C^w + H^w + M^w$\\[1mm]
   \STATE \qquad \textbf{if} $L=0$ \textbf{then} $\theta = \max(w: w<\theta, \; w \in \mathcal{W})$\\[1mm]

   \STATE \qquad set $n_0=L$ and notify each entity $i$ to set $\mathcal{W}_i = \mathcal{L}_i$\\[1mm]
   \STATE \quad \textbf{else}\\[1mm]
   \STATE \qquad set $n_0=H$ and notify each entity $i$ to set $\mathcal{W}_i = \mathcal{H}_i$\\[1mm]
   \ENDWHILE\\[2mm]

   \STATE  \textbf{Center}: \quad $m_0=\frac{1-C\frac{1}{\epsilon n}}{1-C^w}$ and broadcasts it\\[2mm]
   \STATE  \textbf{Each entity $i$}: \quad set each coordinate as
   $w'_i = \left\{ \begin{array}{ll} \frac{1}{\epsilon n}  & \,\text{if  } w_i > \theta \\
                                    w_i m_0              & \,\text{if } w_i \le \theta
   \end{array} \right.$

\end{algorithmic}
\end{algorithm}

\begin{algorithm}[!ht]
   \caption{Distributedly finding the median}
   \label{alg:median}
\begin{algorithmic}
\vskip0.2em

   \STATE {\bf Input:}\\[1.5mm]
   \quad \textbf{Each entity $i$}: a disjoint subset $\mathcal{W}_i$ of $\mathcal{W}=\{w_1,\dots,w_n\}$\\[2mm]
   \STATE  \textbf{Each entity $i$}: Send the median $m_i$ of $\mathcal{W}_i$ to the center \\[1.5mm]

   \WHILE {$|\mathcal{W}_i|>1$ for some $i\in[k]$}
   \vskip1.5mm
   \STATE  \textbf{Center}: \\[1.5mm]
   \quad Find the maximum and minimum of the $k$ medians, denoted by $\overline{m}$ and $\underline{m}$ and \\
    \quad notify the corresponding entities, denoted by $A$ and $B$.\\[1mm]
   
   
   \STATE  \textbf{Entity $\;A\;$}: Send $\; \overline{n}=|i: w_i\in \mathcal{W}_A, \; w_i \ge \overline{m}|\;$ to the center\\[1mm]
   \STATE  \textbf{Entity $\;B\;$}: Send $\; \underline{n}=|i: w_i\in \mathcal{W}_B, \; w_i < \underline{m}|\;$ to the center\\[1.5mm]

   \STATE  \textbf{Center}: Send $r=\min\{\overline{n}, \underline{n}\}$ to entity $A$ and $B$\\[1.5mm]
   \STATE \textbf{Entity $A$}: Remove the largest $r$ elements in $\mathcal{W}_A$\\[1.5mm]
   \STATE \textbf{Entity $B$}: Remove the smallest $r$ elements in $\mathcal{W}_B$\\[1.5mm]
   \STATE  \textbf{Entity $A$ and $B$}: Send the new median to the center\\[1.5mm]
   \ENDWHILE
\end{algorithmic}
\end{algorithm}

\begin{theorem}\label{thm:project}
Algorithm~\ref{alg:project} projects a $n$-dimensional distribution into the set of all $\epsilon$-smooth distributions $\mathcal{P}$ with $O(k\log^2(n))$ words of total communication complexity.
\end{theorem}
\begin{proof}
Since Algorithm~\ref{alg:project} is a direct adaptation of the centralized projection algorithm in \cite{Herbster01}, we omit the proof of its correctness. Because we use a binary search over possible thresholds,  the algorithm runs at most $O(\log(n))$ iteration. Therefore, it suffices to show that the communication complexity of finding the median is at most $O(k\log n)$.
This can be done by the iterative procedure shown in Algorithm~\ref{alg:median}.
Each entity first sends its own median to the center. The center identifies the maximum and minimum local medians, denoted as $\overline{m}$ and $\underline{m}$, respectively. The global median must be between $\overline{m}$ and $\underline{m}$, and removing the same number of elements larger than or equal to $\overline{m}$ and less than $\underline{m}$ will not change the median. Therefore, the center can notify the two corresponding entities and let them remove the same number of elements. At least one entity will reduce its size by half, so the algorithm stops after $O(k\log n)$ iterations. Note that except for the first round, we only need to communicate the updated medians of two entities at each round, so the overall communication complexity is $O(k\log n)$ words.

In practice, it is often easier and more efficient to use a \textit{quickselect}-based distributed algorithm to find the median. The idea is to randomly select and broadcast a weight at each iteration. This, in expectation, can remove half of the possible median candidates. This approach achieves the same communication complexity in expectation.
\end{proof}

The complete distributed agnostic boosting algorithm is shown in Algorithm~\ref{alg:dist_agboost}.
We summarize our theoretical results in the next Theorem.

\begin{theorem}
Given access to a $\beta$-weak agnostic learner, Algorithm~\ref{alg:dist_agboost} achieves error rate $\frac{2err_D(H)}{1/2-\beta}+\epsilon$ by using at most $O(\frac{\log(1/\epsilon)}{(1/2-\beta)^2})$ rounds, each involving $O((d/\beta^2)\log(1/\beta))$ examples and an additional $O(kd\log^2(\frac{d\log(1/\epsilon)}{(1/2-\beta)\epsilon}))$ words of communication per round.
\end{theorem}

\begin{algorithm}[!ht]
   \caption{Distributed agnostic boosting algorithm}
   \label{alg:dist_agboost}
\begin{algorithmic}
\vskip0.2em

   \STATE {\bf Initialization:}\\[1mm]
   \STATE  \textbf{Center}: Access to a $\beta$-agnostic weak learner. Set $\gamma = \frac{1}{2}(\frac{1}{2}-\beta)$\\[1mm]
   \STATE  \textbf{Each entity $i$}:\\[1mm]
   \quad Sample $S_i$ drawn from $D_i$ such that $S=\cup_i S_i$ with size $n=\tilde{O}(\frac{\log(1/\epsilon)d}{(1/2-\beta)^2\epsilon^2}))$ \\[1mm]
   \quad Set weights $v^{(1)}_{i,x}=1/|S_i|$ for each $(x,y) \in S_i$\\[1mm]
   \FOR{$t=1,2,\dots,T$}
   \vskip1mm
\STATE  \textbf{Each entity $i$}: Send $w^{(t)}_i =\sum_{x\in S_i} v^{(t)}_{i,x}$ to the center\\[1mm]
\STATE  \textbf{Center}: Let $W^{(t)}=\sum_i w^{t}_i$. Determine the number of examples $n^{(t)}_i$ to request from each\\
\qquad entity $i$ by sampling $O(\frac{d}{\beta^2}\log\frac{1}{\beta})$ times from the multinomial distribution $w^{(t)}_i/W^{(t)}$,\\
\qquad  and then send each number $n^{(t)}_i$ to entity $i$.\\[1mm]
   \STATE  \textbf{Each entity $i$}: sample $n^{(t)}_i$ times from $S_i$ proportional to $v^{(t)}_{i,x}$ and send them to the center
   \STATE  \textbf{Center}: run the $\beta$-agnostic weak learner on the union of the received $O(\frac{d}{\beta^2}\log\frac{1}{\beta})$ examples,\\
    \qquad and then broadcast the returned hypothesis $h^{(t)}$
   \STATE  \textbf{Each entity $i$}: update the weight of each example $(x,y)$
   \[
   \hat{v}^{(t+1)}_{i,x} = \left\{ \begin{array}{ll} v^{(t)}_{i,x} (1-\gamma) & \,\text{if  } h^{(t)}(x)=y \\
     v^{(t)}_{i,x} & \text{otherwise}

   \end{array} \right.
   \]
   \vskip1mm
   Distributedly normalize and then project the weights by Algorithm~\ref{alg:project}\\[1mm]
   \ENDFOR
    \STATE {\bf Output:} The hypothesis $\mathrm{sign}\left(\frac{1}{T}\sum_{t=1}^T h^{(t)}\right)$
\end{algorithmic}
\end{algorithm}

\begin{proof}
The boosting algorithm starts by drawing from $D$ a sample $S$ of size $n=\tilde{O}(\frac{\log(1/\epsilon)d}{(1/2-\beta)^2\epsilon^2})$ across the $k$ entities without communicating them. If $S$ is a centralized dataset, then by Theorem~\ref{thm:smooth-agnostic} we know that Algorithm~\ref{alg:mwu} achieves error rate $\frac{2err_{S}(H)}{1/2-\beta} + \frac{\epsilon}{2}$ on $S$ using $O(\frac{\log(1/\epsilon)}{(1/2-\beta)^2})$ iterations.
We have shown that Algorithm~\ref{alg:dist_agboost} is a correct simulation of Algorithm~\ref{alg:mwu} in the distributed setting, and thus we achieve the same error bound on $S$. The number of communication rounds is the same as the number of iterations of the boosting algorithm. And in each round, the communication includes $O(d/\beta^2\log(1/\beta))$ examples for finding the $\beta$-weak hypothesis, $O(kd)$ words for broadcasting the hypothesis and some numbers, and $O(k\log^2 (n))$ words for the distributed Bregman projection.

So far we only have the error bound of $\frac{2err_S(H)}{1/2-\beta} + \frac{\epsilon}{2}$ on $S$.
To obtain the generalization error bound, note that with $n=\tilde{O}(\frac{\log(1/\epsilon)d}{(1/2-\beta)^2\epsilon^2})$ and by the standard VC-dimension argument,
we have that with high probability $err_S(H)\le err_D(H) + \frac{(1/2-\beta)\epsilon}{8}$, and the generalization error of our final hypothesis deviates from the empirical error by at most $\epsilon/4$, which completes the proof with the desired generalization error bound.
\end{proof}

\section{Experiments}
In this section, we compare the empirical performance of the proposed distributed boosting algorithms with two other algorithms on synthetic and real-world datasets.
The first one is distributed AdaBoost \cite{BalcanBFM12}, which is similar to our algorithm but without the projection step. The second one is the distributed logistic regression
algorithm available in the MPI implementation of the Liblinear package \cite{zhuang2014distributed}.
We choose it as a comparison to a non-boosting approach. Note that Liblinear is a highly-optimized package while our implementation is not, so the comparison in terms of speed is not absolutely fair. However, we show that our approach, grounded in a rigorous framework, is comparable to this leading method in practice. 

\subsection{Experiment setup}
All three algorithms are implemented in C using MPI, and all the experiments are run on Amazon EC2 with 16 m3.large machines.
The data is uniformly partitioned across 16 machines.
All the results are averaged over 10 independent trials. Logistic regression is a deterministic algorithm, so we do not show the standard deviation
of the error rate. We however still run it for 10 times to get the average running time.
Since each algorithm has different number of parameters, for fairness, we do not tune the parameters.
For the two boosting algorithms, we use $T=100$ decision stumps as our weak learners and set $\beta=0.2$ and $\epsilon=0.1$ in all experiments.
For logistic regression, we use the default parameter $C=1$.

\subsection{Synthetic dataset}
We use the synthetic dataset from \cite{long2008random}. 
This dataset has an interesting theoretical property that although it is linearly separable, by randomly flipping a tiny fraction of labels, all convex potential boosting algorithms, including AdaBoost, fail to learn well.
A random example is generated as follows. The label $y$ is randomly chosen from $\{-1, +1\}$ with equal odds.
The feature $x=\left< x_1, \dots, x_{21} \right>$, where $x_i \in \{-1, +1\}$, is sampled from a mixture distribution:
\textbf{1)} With probability $1/4$, set all $x_i$ to be equal to $y$. \textbf{2)} With probability $1/4$, set $x_1 = x_2 = \dots = x_{11} = y$ and $x_{12} = x_{13} = \dots = x_{21} = -y$.
\textbf{3)} With probability $1/2$, randomly set 5 coordinates from the first 11 and 6 coordinates from the last 10 to be equal to $y$. Set the remaining coordinates to $ -y$.

We generate 1,600,000 examples in total for training on 16 machines and test on a separate set of size 100,000.
The results are shown in Table~\ref{table:result-synthetic}.
One can see that our approach (Dist.SmoothBoost), is more resistant to noise than Dist.AdaBoost and significantly outperforms it for having upto $1\%$ noise. In high noise setting ($10\%$), Liblinear performs poorly, while our approach achieves the best error rate.

\begin{table}[t]
\caption{Average (over 10 trials) error rate (\%) and standard deviation on the synthetic dataset}
\label{table:result-synthetic}
\vskip 0.15in
\begin{center}
\begin{small}
\begin{tabular}{rrrr}
\toprule
\textbf{Noise}  & \textbf{Dist.AdaBoost}  &   \textbf{Dist.SmoothBoost} &  \textbf{Liblinear-LR} \\
\midrule
0.1\%   & 11.64 $\pm$ 3.82   &   4.28 $\pm$ 0.66  & 0.00 \\
1\%     &  25.97 $\pm$ 1.56    &   13.38 $\pm$ 4.66   & 0.00   \\
10\%    &  28.04 $\pm$ 0.94    &   27.07 $\pm$ 1.60   &  37.67 \\
\bottomrule
\end{tabular}
\end{small}
\end{center}
\vskip -0.2in
\end{table}

\begin{table}[t]
\caption{Average (over 10 trials) error rate (\%) and standard deviation on real-world datasets}
\label{table:result}
\vskip 0.15in
\begin{center}
\begin{small}
\begin{tabular}{lrr|rrr}
 \toprule
\textbf{Dataset} & \textbf{$\#$examples}  & \textbf{$\#$ features} & \textbf{Dist.AdaBoost}  &   \textbf{Dist.SmoothBoost} &  \textbf{Liblinear-LR} \\
\midrule
Adult  &  48,842  &  123  & 15.71 $\pm$ 0.16   &   \textbf{15.07} $\pm$ 2.32  & 15.36 \\
Ijcnn1 &   141,691   &  22  &  5.90 $\pm$ 0.10   &   \textbf{4.33} $\pm$ 0.18 & 7.57   \\
Cod-RNA &   488,565   &   8  &  \textbf{6.12} $\pm$ 0.09   &   6.51 $\pm$ 0.11  & 11.79  \\
Covtype &   581,012   &  54  &  24.98 $\pm$ 0.22   &    24.68 $\pm$ 0.30 & \textbf{24.52}   \\
Yahoo  &  3,251,378   & 10  &  37.08 $\pm$ 0.15  &   \textbf{36.86} $\pm$ 0.27 & 39.15   \\
\bottomrule

\end{tabular}
\end{small}
\end{center}
\vskip -0.1in
\end{table}

\subsection{Real-world datasets}
We run the experiments on 5 real-world datasets with sizes ranging from 50 thousands to over 3 millions: \textsc{Adult}, \textsc{Ijcnn1}, \textsc{Cod-RNA}, and \textsc{Covtype}
from the LibSVM data repository \footnote{\url{http://www.csie.ntu.edu.tw/~cjlin/libsvmtools/datasets}.}; \textsc{Yahoo} from the Yahoo! WebScope dataset \cite{chu2009case}. The Yahoo dataset is used for predicting whether a user will click the news article on their front page. It contains user click logs and is extremely imbalanced. We trim down this dataset so that the number of positive and negative examples are the same. The detailed information of the datasets are summarized in Table~\ref{table:result}. Each dataset is randomly split into 4/5 for the training set and 1/5 for the testing set. 

The average error rate and the total running time are summarized in Table~\ref{table:result} and Table~\ref{table:time}, respectively.
The bold entries indicates the best error rate. Our approach outperforms the other two on 3 datasets and performs competitively on the other 2 datasets.
In terms of running time, Liblinear is the fastest on all datasets.
However, the communication of our algorithm only depends on the dimension $d$, so even for the largest dataset (\textsc{Yahoo}), it can still finish within 4 seconds.
Therefore, our algorithm is suitable for many real-world situations where the number of examples is much larger than the dimension of the data.
Furthermore, our algorithm can be used with more advanced weak learners, such as distributed logistic regression, to further reduce the running time.

\begin{table}[!t]
\footnotesize
\caption{Average run time (sec) on real-world datasets}
\label{table:time}
\begin{center}
\begin{tabular}{lrrr}
\toprule
\textbf{Dataset} 
& \textbf{Dist.AdaBoost} 
& \textbf{Dist.SmoothBoost}
& \textbf{Liblinear-LR}
\\
\midrule
Adult    & 5.02   &  15.54  & 0.06 \\

Ijcnn1   & 0.76   &   9.19 &  0.10   \\

Cod-RNA  & 1.08   &   10.11  & 0.12  \\

Covtype  & 3.71   &   6.48 &  0.31   \\

Yahoo    &  3.37  &   3.79 & 1.37   \\
\bottomrule

\end{tabular}
\end{center}
\vskip -0.2in
\end{table}

\section{Conclusions}
We propose the first distributed boosting algorithm that enjoys strong performance guarantees, being simultaneously noise tolerant, communication efficient, and computationally efficient; furthermore, it is quite flexible in that it can used with a variety  of weak learners. This improves over the prior work of~\cite{BalcanBFM12,Daume12} that were either communication efficient only in noise-free scenarios or computationally prohibitive.
While enjoying nice theoretical guarantees, our algorithm also shows promising empirical results on large synthetic and real-world datasets.

Finally, we raise some related open questions. In this work we assumed a star topology, i.e., the center can communicate with all players directly. An interesting open question is to extend our results to general communication topologies. Another concrete open question is reducing the constant in our error bound  while maintaining good communication complexity. Finally, our approach uses centralized weak learners for learning general concept classes, so the computation is mostly done in the center. Are there efficient distributed weak learners for some specific concept classes? That could provide a more computation balanced distributed learning procedure that enjoys strong communication complexity as well.

\subsubsection*{Acknowledgments}
This work was supported in part by NSF grants  CCF-1101283, CCF-1451177, CCF-1422910, TWC-1526254, IIS-1217559, IIS-1563816, ONR grant N00014-09-1-0751, and AFOSR grant FA9550-09-1-0538.
We also thank Amazon's AWS in Education grant program for providing the Amazon Web Services.
We thank Vitaly Feldman for useful discussions and valuable comments.

\bibliographystyle{plain}
\bibliography{distboost}

\appendix
\section{Proof of Theorem~\ref{thm:smooth-realizable}}

{
\def\thetheorem{\ref{thm:smooth-realizable}}
\begin{theorem}
Given a sample $S$ and access to a $\gamma$-weak learner, Algorithm~\ref{alg:mwu} makes at most $T=O(\frac{\log(1/\epsilon)}{\gamma^2})$ calls to the weak learner with $\epsilon$-smooth distributions and achieves error rate $\epsilon$ on $S$.
\end{theorem}
\addtocounter{theorem}{-1}
}

\begin{proof}
The analysis is based on the well-studied online learning from experts problem.
In each round $t$, the learner has to make a decision based on the advice of $n$ experts. More precisely, the learner chooses a distribution $D^{(t)}$ from a convex feasible set $\mathcal{P}$ and follows the advice of the $i$-th expert with probability $D^{(t)}(i)$. Then, the losses of each expert's suggested actions are revealed as a vector $\ell^{(t)}$. The expected loss of the learner incurred by using $D^{(t)}$ is thus $\ell^{(t)}D^{(t)}$. The goal is to achieve a total expected loss $\sum_{t=1}^T{\ell^{(t)}D^{(t)}}$ not too much more than $\min_{D\in\mathcal{P}}{\sum_{t=1}^T{\ell^{(t)}D}}$, the cost of always using the best fixed distribution in $\mathcal{P}$.
Step 2 and 3 of Algorithm~\ref{alg:mwu}, which is also known as the multiplicative weights update algorithm, has the following regret bound~\cite{Herbster01}.

\begin{lemma}\label{lem:mwu}
For any $0<\gamma\le 1/2$ and any positive integer $T$, the multiplicative weights update algorithm generates distributions $D^{(1)},\dots,D^{(T)}\in\mathcal{P}$ where each $D^{(t)}$ is computed only based on $\ell^{(1)},\dots,\ell^{(t-1)}$, such that for any $D\in\mathcal{P}$,
\[
\sum_{t=1}^T{\ell^{(t)}D^{(t)}} \le (1+\gamma)\sum_{t=1}^T{\ell^{(t)}D} + \frac{RE(D \parallel D^{(1)})}{\gamma},
\]
where, for two distributions $p$ and $q$, the relative entropy $RE(p \parallel q)=\sum_i{p_i\ln(p_i/q_i)}$.
\end{lemma}

To use the above result in boosting, we can think of the $n$ examples in sample $S$ as the set of experts. The learner's task is thus to choose a distribution $D^{(t)}\in\mathcal{P}$ over the sample at each round. The loss $\ell^{(t)}_i$ is defined to be $\mathbf{1}[h^{(t)}(x_i)=y_i]$, where $h^{(t)}$ is the hypothesis returned by the weak learner. To ensure that the boosting algorithm always generates a ``smooth'' distribution, we set the feasible set $\mathcal{P}$ to be the set of all $\epsilon$-smooth distributions. Below we show how this can be applied in boosting, as suggested by \cite{Kale07}.

By the assumption of the $\gamma$-weak learner, we have
\[
\ell^{(t)}D^{(t)}  = \sum_i D^{(t)}(i) \mathbf{1}[h^{(t)}(x_i)=y_i]
  \ge 1/2 + \gamma.
\]
After $T=\lceil \frac{2\ln(1/\epsilon)}{\gamma^2} \rceil +1$ rounds, we set the final hypothesis $f=\mathrm{sign}(\frac{1}{T}\sum_{t=1}^T h^{(t)})$. Let $E\subseteq S$ be the set of examples where $f$ predicts incorrectly. Suppose $|E|\ge \epsilon n$. Let $D=u_E$, the uniform distribution on $E$ and $0$ elsewhere. It is easy to see that $u_E\in\mathcal{P}$, since $|E|\ge \epsilon n$. For each example $(x_i, y_i)\in E$, we have
\[
\sum_{t=1}^T \ell^{(t)}_i = \sum_{t=1}^T\mathbf{1}[h^{(t)}(x_i)=y_i] \le T/2,
\]
since $f$ misclassifies $(x_i, y_i)$. Therefore, $\sum_{t=1}^T \ell^{(t)} u_E \le T/2$. Furthermore, since $|E| \ge\epsilon n$, we have
\[
RE(D \parallel D^{(1)}) = RE(u_E \parallel D^{(1)}) \le \ln(1/\epsilon).
\]
By plugging these facts into the inequality in Lemma~\ref{lem:mwu}, we get
\[
(1/2+\gamma)T \le (1+\gamma)T/2 + \frac{\ln(1/\epsilon)}{\gamma},
\]
which implies $T\le\frac{2\ln(1/\epsilon)}{\gamma^2}$, a contradiction.
\end{proof}

\end{document}